\documentclass[10pt,twocolumn,letterpaper]{article}

\usepackage[top=0.7in, bottom=0.7in, left=0.6in, right=0.7in, columnsep=0.35in]{geometry}

\usepackage{blindtext, graphicx}

\usepackage[cmex10]{amsmath}
\usepackage{amssymb}
\usepackage{amsthm}

\newcommand{\sign}{\mathop{\text{sign}}}

\newtheorem{thm}{Theorem}
\newtheorem{defn}{Definition}
\newtheorem{lemma}{Lemma}
\newtheorem{cor}{Corollary}
\newtheorem{prop}{Proposition}
\newtheorem{remark}{Remark}

\usepackage{caption}

\begin{document}
\title{Information Theoretic Limits for Linear Prediction with Graph-Structured Sparsity}

\author{Adarsh Barik \\
Krannert School of Management\\
Purdue University\\
West Lafayette, Indiana 47906\\
Email: abarik@purdue.edu
\and
Jean Honorio \\
Department of Computer Science\\
Purdue University\\
West Lafayette, Indiana 47906\\
Email: jhonorio@purdue.edu
\and
Mohit Tawarmalani \\
Krannert School of Management\\
Purdue University\\
West Lafayette, Indiana 47906\\
Email: mtawarma@purdue.edu}
\date{}

\maketitle

\begin{abstract}
%\boldmath
We analyze the necessary number of samples for sparse vector recovery in a noisy linear prediction setup. This model includes problems such as linear regression and classification. We focus on structured graph models. In particular, we prove that sufficient number of samples for the weighted graph model proposed by Hegde and others~\cite{hegde2015nearly} is also necessary. We use the Fano's inequality~\cite{Cover06} on well constructed ensembles as our main tool in establishing information theoretic lower bounds.
\end{abstract}


\section{Introduction}
Sparse vectors are widely used tools in fields related to high dimensional
data analytics such as machine learning, compressed sensing and statistics.
This makes estimation of sparse vectors an important field of research. In a compressive sensing setting, the problem is to closely approximate a $d-$dimensional signal by an $s-$sparse vector without losing much information. For regression, this is usually done by observing the inner product of the signal with a design matrix. It is a well known fact that if the design matrix satisfies the Restricted Isometry Property (RIP) then estimation can be done efficiently with a sample complexity of $O(s \log \frac{d}{s})$. Many algorithms such as CoSamp~\cite{needell2009cosamp}, Subspace Pursuit (SP)~\cite{dai2008subspace} and Iterative Hard Thresholding (IHT)~\cite{blumensath2009iterative} provide high probability performance guarantees. Baraniuk and others~\cite{baraniuk2010model} came up with a model based sparse recovery framework. Under this framework, the sufficient number of samples for correct recovery is logarithmic with respect to the cardinality of the sparsity model. 

A major issue with the model based framework is that it does not provide any recovery algorithm on its own. In fact, it is some times very hard to come up with an efficient recovery algorithm. Addressing this issue, Hegde and others~\cite{hegde2015nearly} came up with a weighted graph model for graph structured sparsity and provided a nearly linear time recovery algorithm. They also analyzed the sufficient number of samples for efficient recovery. In this paper, we will provide the necessary condition on the sample
complexity for sparse recovery on a weighted graph model. We will also note that our information theoretic lower bound can be applied not only to linear regression but also to other linear prediction tasks such as classification.

The paper is organized as follows. We describe our setup in Section~\ref{sec:model}. Then we briefly describe the weighted graph model in Section~\ref{sec:wgm}. We state our results in Section~\ref{sec:results}. In Section~\ref{sec:specific example}, we apply our technique to some specific examples. At last, we provide some concluding remarks in Section~\ref{sec:conclusion}.

\section{Linear Prediction Model}
\label{sec:model}
In this section, we introduce the observation model for linear prediction and later specify how to use it for specific problems such as linear regression and classification. 
Formally, the problem is to estimate an $s-$sparse vector $\bar{\beta}$ from noisy observations of the form, 
\begin{align}
\label{eq:linreg}
z = f(X \bar{\beta} + e)\ ,
\end{align}
\noindent
where $z \in \mathbb{R}^n$ is the observed output, $X \in \mathbb{R}^{n \times d}$ is the design matrix , $e \in \mathbb{R}^n $ is a noise vector and $f:\mathbb{R}^n\rightarrow \mathbb{R}^n$ is a fixed function. Our task is to recover $\bar{\beta} \in \mathbb{R}^d$ from the observations $z$.

\vspace*{-0.5pt}
\subsection{Linear Regression} 
\label{subsec:linreg}
Linear regression is a special case of the above by choosing $f (x) = x$. Then we simply have,
\begin{align}
\label{eq:linreg1}
z = X \bar{\beta} + e \ .
\end{align}
\noindent
Prior work analyzes the sample complexity of sparse recovery for the linear regression setup. In particular, if the design matrix $X$ satisfies the Restricted Isometry Property (RIP) then algorithms such as CoSamp~\cite{needell2009cosamp}, Subspace Pursuit (SP)~\cite{dai2008subspace} and Iterative Hard Thresholding (IHT)~\cite{blumensath2009iterative} can recover $\bar{\beta}$ quite efficiently and in a stable way with a sample complexity of $O(s \log\frac{d}{s})$. Furthermore, it is known that Gaussian random matrices (or sub-Gaussian in general) satisfy RIP~\cite{baraniuk2008simple}. If we choose our design matrix to be a Gaussian matrix and we have a good sparsity model that incorporates extra information on the sparsity structure then we can reduce the sample complexity to $O(\log m_s)$ where $m_s$ is number of possible supports in the sparsity model, i.e., the cardinality of the sparsity model~\cite{baraniuk2010model}. In the same line of work, Hegde and others~\cite{hegde2015nearly} proposed a weighted graph based sparsity model to efficiently learn $\bar{\beta}$.  

\vspace*{-0.5pt}
\subsection{Classification}
We can model binary classification problems by choosing $f(x) = \sign(x)$ or in other words, we can have,
\begin{align}
\label{eq:linreg2}
z = \sign(X \bar{\beta} + e) \ .
\end{align}
Similar to the linear regression setup, there is also prior work~\cite{gupta2010sample}, \cite{ai2014one}, \cite{gopi2013one}, on analyzing the sample complexity of sparse recovery for binary classification problem (also known as 1-bit compressed sensing). 

Since arguments for establishing information theoretic lower bounds are not algorithm specific, we can extend our basic argument to the both settings mentioned above. For comparison, we will use the results by Hegde and others~\cite{hegde2015nearly} in a linear regression setup. 

\section{Weighted Graph Model (WGM)}
\label{sec:wgm}
In this section, we introduce the Weighted Graph Model (WGM) and formally state the sample complexity results from~\cite{hegde2015nearly}. The Weighted Graph Model is defined on an underlying graph $G = (V,E)$ whose vertices are on the coefficients of the unknown $s-$sparse vector $\bar{\beta} \in \mathbb{R}^d$ i.e. $V = [d] = \{1, 2, \dots, d\}$. Moreover, the graph is weighted and thus we introduce a weight function $w : E \rightarrow \mathbb{N}$. Borrowing some notations from~\cite{hegde2015nearly},  for a forest $F \subseteq G$ we denote $\sum_{e \in F} w_e$ as $w(F)$. $B$ denotes the weight budget, $s$ denotes the sparsity (number of non-zero coefficients) of $\bar{\beta}$ and $g$ denotes the number of connected components in $F$. The weight-degree $\rho(v)$ of a node $v \in V$ is the largest number of adjacent nodes connected by edges with the same weight, i.e.,
\begin{align}
\label{eq:weightdegree}
\rho(v) = \max_{b \in N} |\{ (v', v) \in E\ | \ w(v',v) = b \}| \ .
\end{align}  
We define the weight-degree of $G$, $\rho(G)$ to be the maximum weight-degree of any $v \in V$. Next, we define the Weighted Graph Model on coefficients of $\bar{\beta}$ as follows:    

\begin{defn}[Definition 1 in~\cite{hegde2015nearly}]
	\label{def:WGM}
	The $(G, s, g, B)-WGM$ is the set of supports defined as
	\begin{align*}
	\mathbb{M} = \left\lbrace S \subseteq [d] \ | \ |S| = s \ and \ \exists\ F \subseteq G\ with\ V_F = S, \right. \\
	\left. \gamma(F) = g,\ w(F) \leq B \right\rbrace\ ,
	\end{align*}
\end{defn} 
where $\gamma(F)$ is number of connected components in a forest $F$. Authors in ~\cite{hegde2015nearly} provide the following sample complexity result for linear regression under their model:
\begin{thm}[Theorem 3 in~\cite{hegde2015nearly}]
	Let $\bar{\beta} \in \mathbb{R}^d$ be in the $(G, s, g, B)-WGM$. Then 
	\begin{align}
	\label{eq:samplecomplexity}
	n = O(s(\log \rho(G) + \log \frac{B}{s}) + g \log \frac{d}{g})
	\end{align}
	i.i.d. Gaussian observations suffice to estimate $\bar{\beta}$. More precisely, let $e \in \mathbb{R}^n$ be an arbitrary noise vector from equation \eqref{eq:linreg1} and $X$ be an i.i.d. Gaussian matrix. Then we can efficiently find an estimate $\hat{\beta}$ such that 
	\begin{align}
	\label{eq:estimatederror}
	\| \bar{\beta} - \hat{\beta} \| \leq C \|e\|\ ,
	\end{align}
	where $C$ is a constant indepenedent of all variables above.
\end{thm}

Notice that in the noiseless case $(e = 0)$, we recover the exact $\bar{\beta}$. We will prove that information-theoretically, the bound on the sample complexity is tight and thus the algorithm of~\cite{hegde2015nearly} is statistically optimal.

\section{Main Results}
\label{sec:results}
In this section, we will state our results for both the noiseless and the noisy case. We establish an information theoretic lower bound on linear prediction problem defined on WGM. We use Fano's inequality~\cite{Cover06} to prove our result by carefully constructing an ensemble, i.e., a WGM. Any algorithm which infers $\bar{\beta}$ from this particular WGM would require a minimum number of samples. Note that the use of restricted ensembles is customary for information-theoretic lower bounds~\cite{santhanam2012information}~\cite{wang2010information}. It follows that in the case of linear regression, the  upper bound on the sample complexity by Hegde and others~\cite{hegde2015nearly} is indeed tight. 

\vspace*{-0.5pt}
\subsection{Noiseless Case}
 Here, we provide a necessary condition on the sample complexity for exact recovery in the noiseless case. More formally, 
\begin{thm}
	\label{thm:mainresult}
	There exists a particular $(G, s, g, B)-WGM$, and a particular set of weights for the entries in the support of $\bar{\beta}$ such that if we draw a $\bar{\beta} \in \mathbb{R}^d$ uniformly at random and we have a data set $\mathcal{S}$ of $n \in o((s-g) (\log \rho(G) + \log \frac{B}{s-g}) + g \log \frac{d}{g} + (s - g) \log \frac{g}{s - g} + s \log 2)$ i.i.d. observations as defined in equation~\eqref{eq:linreg} with $e = 0$ then $P(\bar{\beta} \neq \hat{\beta}) \geq \frac{1}{2}$ irrespective of the procedure we use to infer $\hat{\beta}$ on $(G,s,g,B)-WGM$ from $\mathcal{S}$.  
\end{thm} 
\begin{proof}[Proof sketch]
	We use Fano's inequality~\cite{Cover06} on a carefully chosen restricted ensemble to prove our theorem. A detailed proof can be found in appendix.
\end{proof}

\vspace*{-0.5pt}
\subsection{Noisy Case}
A similar result can be stated for the noisy case. However, in this case recovery is not exact but is sufficiently close in $l_2$-norm with respect to noise in the signal. Another thing to note is that in~\cite{hegde2015nearly} inferred $\hat{\beta}$ can come from a slightly bigger WGM model but here we actually infer $\hat{\beta}$ from the same WGM. 
\begin{thm}
	\label{thm:noisymainresult}
	There exists a particular $(G, s, g, B)-WGM$, and a particular set of weights for the entries in the support of $\bar{\beta}$ such that if we draw a $\bar{\beta} \in \mathbb{R}^d$ uniformly at random and we have a data set $\mathcal{S}$ of $n \in o((s-g) (\log \rho(G) + \log \frac{B}{s-g}) + g \log \frac{d}{g} + (s-g) \log \frac{g}{s-g} + s \log 2)$ i.i.d. observations as defined in equation~\eqref{eq:linreg} with $e_i \stackrel{iid}{\sim} \mathcal{N}(0,\sigma), \forall i \in \{1 \dots n \}$ then  $ \mathbb{P}(\| \bar{\beta} - \hat{\beta}  \| \geq C \| e \|) \geq \frac{1}{10}$ for $0 < C \leq C_0$ irrespective of the procedure we use to infer $\hat{\beta}$ on $(G,s,g,B)-WGM$ from $\mathcal{S}$.
\end{thm}
\begin{remark}
	Note that when $s \gg g$ and $B \geq s-g$ then $\Omega((s-g) (\log \rho(G) + \log \frac{B}{s-g}) + g \log \frac{d}{g} + (s - g) \log \frac{g}{s - g} +  s\log 2) $ is roughly $\Omega(s(\log \rho(G) + \log \frac{B}{s}) + g \log \frac{d}{g})$.
\end{remark}
\begin{proof}
	We will prove this result in three steps. First, we will carefully construct an underlying graph $G$ for the WGM. Second, we will
	bound mutual information between $\bar{\beta}$ and $\mathcal{S}$ by bounding the Kullback-Leibler (KL) divergence. Third, we will bound the size of properly defined restricted ensemble to complete our proof.
	\begin{figure*}[!t]
		%\centering
		\includegraphics[width=0.9\textwidth]{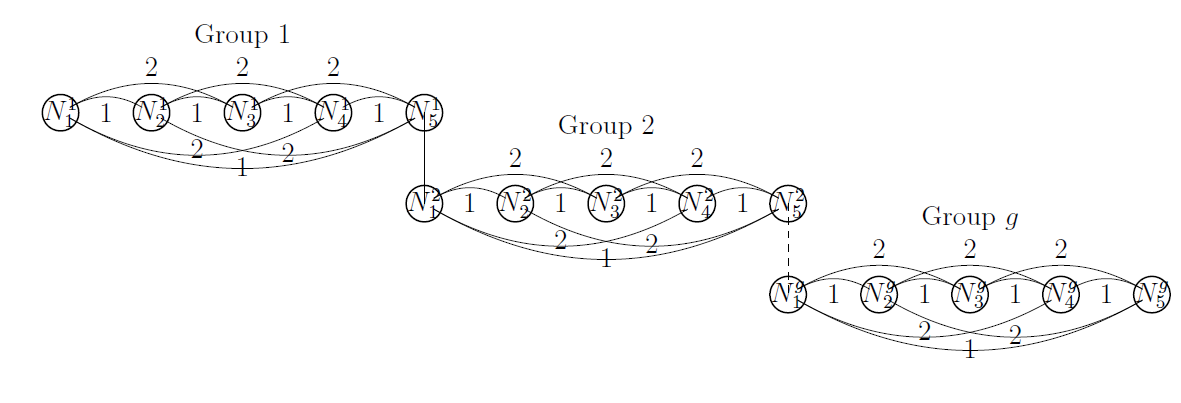}
		\centering \caption{An example of constructing an underlying graph for $\rho(G) = 2$ and $\frac{B}{s-g} = 2$}
		\label{figgraphconst}
	\end{figure*}
	\begin{figure*}[!t]
		%\centering
		\includegraphics[width=0.9\textwidth]{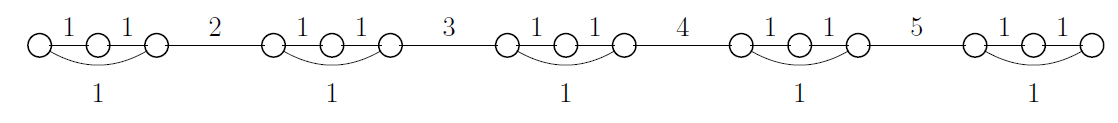}
		\centering\caption{An example of an underlying graph $G$ for $(G, s, g, B)-WGM$ with parameters $d=15, s=10,  g=5, B=5, \rho(G) = 2$}
		\label{figgraphex}
	\end{figure*}
	\noindent
	\paragraph{Constructing an underlying graph $G$ for the WGM} We construct an underlying graph for the WGM using the following steps:
	\begin{itemize}
		\item Divide $d$ nodes equally into $g$ groups with each group having $\frac{d}{g}$ nodes.
		\item For each group $j$, we denote a node by $N_i^j$ where $j$ is the group index and $i$ is the node index. Each group $j$, contains nodes from $N_1^j$ to $N_{\frac{d}{g}}^j$.
		\item We allow for circular indexing, i.e., a node $N_i^j$ where $i > \frac{d}{g}$ is same as node $N_{i - \frac{d}{g}}^j$.  
		\item For each $p = 1,\dots,\frac{B}{s-g}$, node $N_i^j$ has an edge with nodes $N_{i+(p -1)\frac{\rho(G)}{2} + 1}^j$ to $N_{i+p \frac{\rho(G)}{2}}^j$ with weight $p$.
		\item Cross edges between nodes in two different groups are allowed as long as they have edge weights greater than $\frac{B}{s-g}$ and they do not affect $\rho(G)$. 
	\end{itemize}
	Figure~\ref{figgraphconst} shows an example of a graph constructed using the above steps. Furthermore, parameters for our $WGM$ satisfy the following requirements:
	\begin{enumerate}
		\item[R1.] $\frac{d}{g} \geq \frac{\rho(G)B}{s-g} + 1$,
		\item[R2.] $\frac{\rho(G)B}{2(s-g)} \geq \frac{s}{g} - 1$,
		\item[R3.] $B \geq s-g$\ .
	\end{enumerate} 
	These are quite mild requirements (see appendix) on the parameters and are easy to fulfill. Figure~\ref{figgraphex} shows one graph which follows our construction and also fulfills R1, R2 and R3. 
	\noindent
	 We define our restricted ensemble $\mathcal{F}$ on $G$ as:
	\begin{align}
	\label{eq:F noisy}
	\begin{split} 
	\mathcal{F} &= \left\lbrace \beta\ |\ \beta_i = 0, \text{ if } i \notin S,\  \beta_i \in \left\lbrace \frac{C_0\sigma \sqrt{d}}{\sqrt{2(1 - \epsilon)}}, \right.\right. \\
	 &\left. \left. \frac{C_0\sigma \sqrt{d}}{\sqrt{2(1 - \epsilon)}} + \frac{C_0\sigma \sqrt{d}}{\sqrt{(1 - \epsilon)}} \right\rbrace,  \text{ if } i \in S,\ S \in \mathbb{M} \right\rbrace\ ,
	\end{split} 
	\end{align}
	for some $0 < \epsilon < 1$ and $\mathbb{M}$ is as in Definition~\ref{def:WGM}.
	
	\noindent
	Our true $\bar{\beta}$ is picked uniformly at random from the above restricted ensemble. We will prove that on this restricted ensemble, our Theorem~\ref{thm:noisymainresult} holds. We will make use of following lemmas for our proof:
	\begin{lemma}
		\label{lemma:ballcovering}
		Given the restricted ensemble $\mathcal{F}$,
		\begin{align*}
		\| \bar{\beta} - \hat{\beta} \| \leq \frac{C_0\sigma \sqrt{d}}{\sqrt{(1-\epsilon)}} \iff \bar{\beta} = \hat{\beta}\ .
		\end{align*}
	\end{lemma}	
	We are dealing with high dimensional cases, hence	moving forward we will assume that $n < d$. We state another lemma:
	\begin{lemma}
		\label{lemma:ebound}
		For some $0 < \epsilon < 1$,
		\begin{align*}
		\mathbb{P}\Big(\|e \|^2 \leq \sigma^2 \frac{n}{1 - \epsilon}\Big) \geq 1 - \exp\Big(-\frac{\epsilon^2 n}{4}\Big)\ .
		\end{align*}
	\end{lemma}
	From Lemma~\ref{lemma:ballcovering}, Lemma~\ref{lemma:ebound} and using the fact that 
	$d > n$ and $C \leq C_0$, the corollary below follows:
	\begin{cor}
		\label{cor: concen bound}
		\begin{align*}
		\mathbb{P}\Big(\| \bar{\beta} - \hat{\beta} \| \geq C \| e \| \ | \ \bar{\beta} \neq \hat{\beta}\Big) \geq 1 - \exp\Big(-\frac{\epsilon^2 n}{4}\Big)\ .
		\end{align*}
	\end{cor}
	\noindent
	\paragraph{Bound on the mutual information}  We will assume that the elements of design matrix $X$ have been chosen at random and independently from $\mathcal{N}(0, 1)$. The linear prediction problem  from Section~\ref{sec:model} can be described by the following Markov's chain:
	\begin{align}
	\label{eq:markov}
	\bar{\beta} \rightarrow y = X \bar{\beta} + e \rightarrow z = f (y ) \rightarrow \hat{\beta} \ .
	\end{align} 
	\noindent
	Lets say $\mathcal{S}$ contains $n$ i.i.d. observations of $z$ and $\mathcal{S}'$ contains $n$ i.i.d. observations of $y$. Then using the data processing inequality~\cite{Cover06} we can say that,
	\begin{align}
	\label{eq:data processing}
	\mathbb{I}(\bar{\beta}, \mathcal{S}) \leq \mathbb{I}(\bar{\beta}, \mathcal{S}')\ .
	\end{align}
	\noindent 
	Hence, for our purpose it suffices to have an upper bound on $\mathbb{I}(\bar{\beta}, \mathcal{S}')$. Now we can bound the mutual information by the following~\cite{Yu97}:
	\begin{align}
	\mathbb{I}(\bar{\beta}, \mathcal{S}') \leq \frac{1}{|\mathcal{F}|^2} \sum_{\beta \in \mathcal{F}}^{} \sum_{\beta' \in \mathcal{F}}^{} \mathbb{KL}(\mathbb{P}_{\mathcal{S}'|\beta} \| \mathbb{P}_{\mathcal{S}'|\beta'} )\ ,
	\end{align}
	where $\mathbb{KL}$ is the Kullback-Leibler divergence. Note that $\mathcal{S}'$ consists of $n$ i.i.d. observations of $y$. Hence,
	\begin{align}
	\label{eq:bound info}
	\mathbb{I}(\bar{\beta}, \mathcal{S}') \leq \frac{n}{|\mathcal{F}|^2} \sum_{\beta \in \mathcal{F}}^{} \sum_{\beta' \in \mathcal{F}}^{} \mathbb{KL}(\mathbb{P}_{y_i|\beta} \| \mathbb{P}_{y_i|\beta'} )
	\end{align}
	\noindent
	Furthermore, from equation~\eqref{eq:markov} and noting that the elements of $X$ come independently from $ \mathcal{N}(0, 1)$,
	\begin{align*}
	&y_i =  X_i \beta + e_i\\
	&y_i|\beta \sim \mathcal{N}(0 , \| \beta \|^2 + \sigma^2 )  \\
	&y_i|\beta' \sim \mathcal{N}(0 , \| \beta' \|^2 + \sigma^2 )\ .
	\end{align*}
	We can bound the Kullback-Leibler divergence between $\mathbb{P}_{y_i|\beta}$ and $\mathbb{P}_{y_i|\beta'}$ as follows:
	\begin{align*}
	\mathbb{KL}(\mathbb{P}_{y_i|\beta} \| \mathbb{P}_{y_i|\beta'}) &= \frac{1}{2}\left(\frac{\|\beta \|^2 + \sigma^2}{\|\beta' \|^2 + \sigma^2} - 1 \right. \\
	&\left. - \log \frac{\|\beta \|^2 + \sigma^2}{\|\beta' \|^2 + \sigma^2} \right) \\
	&\leq \frac{1}{2}\Big(\frac{\|\beta \|^2 + \sigma^2}{\|\beta' \|^2 + \sigma^2} - 1 - 1 + \frac{\|\beta' \|^2 + \sigma^2}{\|\beta \|^2 + \sigma^2} \Big) \\
	&\leq \frac{1}{2}\left(\frac{(\frac{C_0\sigma \sqrt{d}}{\sqrt{2(1-\epsilon)}} + \frac{C_0\sigma\sqrt{d}}{\sqrt{(1-\epsilon)}})^2 s  + \sigma^2}{(\frac{C_0\sigma\sqrt{d}}{\sqrt{2(1-\epsilon)}})^2 s + \sigma^2} \right. \\
	&\left. + \frac{(\frac{C_0\sigma\sqrt{d}}{\sqrt{2(1-\epsilon)}} + \frac{C_0\sigma\sqrt{d}}{\sqrt{(1-\epsilon)}} )^2 s + \sigma^2}{(\frac{C_0\sigma\sqrt{d}}{\sqrt{2(1-\epsilon)}} )^2 s + \sigma^2} - 2 \right) \\
	&\leq \frac{1}{2}\Big(2 \frac{(\frac{C_0\sigma\sqrt{d}}{\sqrt{2(1-\epsilon)}} + \frac{C_0\sigma\sqrt{d}}{\sqrt{1-\epsilon}})^2 s  + \sigma^2}{(\frac{C_0\sigma\sqrt{d}}{\sqrt{2(1-\epsilon)}})^2 s + \sigma^2} - 2 \Big) \\
	&\leq \frac{1}{2}\Big(2 \frac{(\sqrt{2}+1)^2(\frac{C_0\sigma\sqrt{d}}{\sqrt{2(1-\epsilon)}})^2 s  + \sigma^2}{(\frac{C_0\sigma\sqrt{d}}{\sqrt{2(1-\epsilon)}})^2 s + \sigma^2} - 2 \Big) \\
	&\leq (\sqrt{2}+1)^2 - 1 \\
	&\leq 5 \ .
	\end{align*}
	The first inequality holds because $1 - \frac{1}{x} \leq \log x, \forall x > 0$, the second inequality holds by taking the largest value of numerators and the smallest value of denominators. The other inequalities follow from simple algebraic manipulation. Substituting $\mathbb{KL}(\mathbb{P}_{y_i|\beta} \| \mathbb{P}_{y_i|\beta'})$ in equation~\eqref{eq:bound info} we get,
	\begin{align}
	\label{eq:mutual info bound}
	\mathbb{I}(\bar{\beta}, \mathcal{S}') \leq  5 n \ .
	\end{align} 

\noindent
\paragraph{Bound on $|\mathcal{F}|$} Now we will count elements in $\mathcal{F}$ to complete our proof. We present the following counting argument to establish a lower bound on all the possible supports for our restricted ensemble:
\begin{enumerate}
	\item We choose one node from each of the $g$ groups in underlying graph $G$ to be root of a connected component. Each group has $\frac{d}{g}$ possible candidates for the root and hence we can choose them in $(\frac{d}{g})^g$ possible ways.
	\item Since we are interested only in establishing a lower bound on $\mathcal{F}$, we will only consider the cases where each connected component has $\frac{s}{g}$ nodes. Moreover, given a root node $N_i^j$ in group $j$, we will choose the remaining $\frac{s}{g} - 1$ nodes connected with the root only from the nodes $N_{i+1}^j$ to nodes $N_{i+\frac{B\rho(G)}{2(s - g)}}^j$ (using circular indices if needed). Construction of the graph $G$ allows us to do this. At least till the last $\frac{\rho(G)B}{2(s - g)}$ nodes, we always include node $N_i^j$ and we never include $N_r^j, r \leq i-1$ in our selection. Furthermore, R1 guarantees that we have enough nodes to avoid any possible repetitions due to circular indices for the last $\frac{\rho(G)B}{2(s -g)}$ nodes and R2 ensures that we have enough nodes to form a connected component. This guarantees that all the supports are unique. Hence, given a root node $N_i^j$ we have ${\frac{\rho(G)B}{2(s-g)} \choose \frac{s}{g} -1} $ choices which across all the groups comes out to be $({\frac{\rho(G)B}{2(s - g)} \choose \frac{s}{g} -1 })^g$.     
	\item Each entry in the support of $\beta$ can take two values which can either be $\frac{C_0\sigma \sqrt{d}}{\sqrt{2(1 - \epsilon)}}$ or $\frac{C_0\sigma \sqrt{d}}{\sqrt{2(1 - \epsilon)}} + \frac{C_0\sigma \sqrt{d}}{\sqrt{(1 - \epsilon)}} $. 
\end{enumerate}
It should be noted that any support chosen using the above steps satisfies constraint on weight budget, i.e., $w(F) \leq B$ as the maximum edge weight in any connected component will always be less than or equal to $\frac{B}{s-g}$. Combining all the above steps together we get:
\begin{align}
\label{eq:bound on F}
\begin{split}
|\mathcal{F}| &\geq 2^s (\frac{d}{g})^g ({\frac{\rho(G)B}{2(s - g)} \choose \frac{s}{g} -1 })^g \\
&\geq 2^s (\frac{d}{g})^g (\frac{\rho(G)Bg}{2(s-g)^2})^{(s-g)} \ . 
\end{split}
\end{align}
%\noindent
Using Fano's inequality~\cite{Cover06} and results from equation~\eqref{eq:mutual info bound} and equation~\eqref{eq:bound on F}, it is easy to prove the following lemma,
\begin{lemma}
	\label{lemma:fano}
	If $n \in o(\log |\mathcal{F}|)$ then $\mathbb{P}(\hat{\beta} \neq \bar{\beta}) \geq \frac{1}{2}$. 
\end{lemma}
%\noindent
By using Bayes' Theorem and combining Corollary~\ref{cor: concen bound} and Lemma~\ref{lemma:fano},
\begin{align}
\label{eq:bound on error}
\begin{split}
\mathbb{P}\Big(\| \bar{\beta} - \hat{\beta} \| \geq C\|e\|\Big) &\geq \mathbb{P}\Big(\| \bar{\beta} - \hat{\beta} \| \geq C\|e\|, \bar{\beta} \neq \hat{\beta}\Big) \\
&= \mathbb{P}\Big(\| \bar{\beta} - \hat{\beta} \| \geq C\|e\| \ | \ \bar{\beta} \neq \hat{\beta}\Big) \\
&\mathbb{P}\Big(\bar{\beta} \neq \hat{\beta}\Big)\\
&\geq \Big(1 - \exp\Big(-\frac{\epsilon^2 n}{4}\Big)\Big) \frac{1}{2} \ . 
\end{split}
\end{align}
The last inequality~\eqref{eq:bound on error} holds when $n$ is $o(\log |\mathcal{F}|)$. We also know that $n \geq 1$ and if we choose $\epsilon \geq \sqrt{-4\log 0.8} \sim 0.9448$, then we can write inequality~\eqref{eq:bound on error} as,
\begin{align}
\label{eq: final res}
\mathbb{P}(\| \bar{\beta} - \hat{\beta} \| \geq C\|e\|) &\geq \frac{1}{10} \ .
\end{align}
\end{proof}

\section{Specific Examples}
\label{sec:specific example}
Here, we will provide counting arguments for some of the well-known sparsity structures, such as tree sparsity and block sparsity models. It should be noted that barring the count of possible supports in the specific model our technique can be used to prove lower bounds of the sample complexity for other sparsity structures. 

\subsection{Tree-structured sparsity model}
\label{subsec:tree sparsity}
The tree-sparsity model~\cite{baraniuk2010model},~\cite{hegde2014fast} is used in many applications such as wavelet decomposition of piecewise smooth signals and images. In this model, we assume that the coefficients of the $s-$sparse signal form a $k-$ary tree and the support of the sparse signal form a rooted and connected sub-tree on $s$ nodes in this $k-$ary tree. The arrangement is such that if a node is part of this subtree then its parent is also included in it. Here, we will discuss the case of a binary tree which can be generalized to a $k-$ary tree. In particular, the following proposition provides a lower bound on the number of possible supports of an $s-$sparse signal following a binary tree-structured sparsity model.
\begin{prop}
\label{prop:tree sparse}
In a binary tree-structured sparsity model $\mathcal{F}$, $\log |\mathcal{F}| \geq cs$ for some $c > 0$. 
\end{prop}   
The proof of the proposition~\ref{prop:tree sparse} follows from the fact that we have at least $2^s$ different choices of $\bar{\beta}$ in our restricted ensemble.   
From the above and following the same proof technique as before, it is easy to prove the following corollary for the noisy case (a similar result holds for the noiseless case as well).
\begin{cor}
\label{cor:tree sparse}
In a binary tree-structured sparsity model, if $n \in o(s)$ then $\mathbb{P}(\| \bar{\beta} - \hat{\beta}  \| \geq C \| e \|) \geq \frac{1}{10}$.
\end{cor}
Essentially, Corollary~\ref{cor:tree sparse} proves that the $O(s)$ sample complexity achieved in~\cite{hegde2015nearly} is optimal for the tree-sparsity model.

\subsection{Block sparsity model}
\label{subsec: block sparse}
In the block sparsity model,~\cite{baraniuk2010model}, an $s-$sparse signal, $\beta \in \mathbb{R}^{J \times N}$, can be represented as a matrix with $J$ rows and $N$ columns. The support of $\beta$ comes from $K$ columns of this matrix such that $s = J K$. More precisely, 
\begin{defn}[Definition 11 in~\cite{baraniuk2010model}]
	\begin{align*}
	\mathcal{S}_K = \left\lbrace \beta = [ \beta_1 \dots \beta_N ]\in \mathbb{R}^{J \times N} \ such\ that\ \right. \\
	\left. \beta_n = 0\ for\ n \notin L,\ L \subseteq \{1,\dots, N\},\ |L|=K \right\rbrace \ .
	\end{align*}
\end{defn} 

The above can be modeled as a graph model. In particular, we can construct a graph $G$ over all the elements in $\beta$ by treating nodes in the column of the matrix as connected nodes (see Fig.~\ref{figblock}) and then our problem is to choose $K$ connected components from $N$.    

\begin{figure}[!t]
	\centering
	\includegraphics[width=1in]{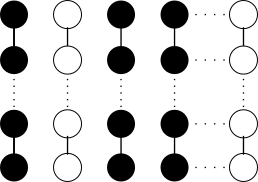}
	\caption{Block sparsity structure as a graph model: nodes are variables, black nodes are selected variables}
	\label{figblock}
\end{figure}

It is easy to see that the number of possible supports in this model, $\mathcal{F}$, would be, $|\mathcal{F}| = 2^{KJ} {N \choose K} \geq 2^{KJ} (\frac{N}{K})^K$. Correspondingly the necessary number of samples for efficient signal recovery comes out to be $\Omega(KJ + K \log \frac{N}{K})$. An upper bound of $O(KJ + K \log \frac{N}{K})$ was derived in~\cite{baraniuk2010model} which matches our lower bound.

\section{Concluding Remarks}
\label{sec:conclusion}
We proved that the necessary number of samples required to efficiently recover a sparse vector in the weighted graph model is of the same order as the sufficient number of samples provided by Hegde and others~\cite{hegde2015nearly}. Moreover, our results not only pertain to linear regression but also apply to linear prediction problems in general.

\vspace{10mm}
\section{APPENDIX}
\section*{Proof of Lemma~\ref{lemma:ballcovering}}
\begin{proof}
	First note that when $\bar{\beta} = \hat{\beta}$ then its obvious that Lemma~\ref{lemma:ballcovering} holds. Here, we will prove that two arbitrarily chosen $\beta_1$ and $\beta_2$ such that $\beta_1, \beta_2 \in \mathcal{F}$ where $\beta_1 \neq \beta_2$ then $\| \beta_1 - \beta_2 \| \geq \frac{C_0\sigma\sqrt{d}}{\sqrt{1-\epsilon}}$. $\mathcal{F}$ is as defined in equation~\eqref{eq:F noisy}.
	\noindent
	\paragraph{$\beta_1$ and $\beta_2$ have the same support} Since we assume that $\beta_1 \neq \beta_2$, thus they must differ in at least one position on their support. Lets say that one such position is $i$. Then,   
	\begin{align*}
	\|\beta_1 - \beta_2\| &\geq |\beta_{1i} - \beta_{2i}| \\
	&=\frac{C_0\sigma \sqrt{d}}{\sqrt{1 - \epsilon}} \ .
	\end{align*}
	\paragraph{$\beta_1$ and $\beta_2$ have different supports} When $\beta_1$ and $\beta_2$ have different supports then we can always find $i$ and $j$ such that $i \in S_1, i \notin S_2$ and $j \notin S_1, j \in S_2$ where $S_1$ and $S_2$ are supports of $\beta_1$ and $\beta_2$ respectively. Then, 
	\begin{align*}
	\| \beta_1 - \beta_2 \| &\geq \sqrt{\beta_{1i}^2 + \beta_{2j}^2 } \\
	&\geq \sqrt{(\frac{C_0\sigma \sqrt{d}}{\sqrt{2(1 - \epsilon)}} )^2+ (\frac{C_0\sigma \sqrt{d}}{\sqrt{2(1 - \epsilon)} })^2 }\\
	&=\frac{C_0\sigma \sqrt{d}}{\sqrt{1 - \epsilon}} \ .
	\end{align*}
	Since this is true for any two arbitrarily chosen $\beta_1$ and $\beta_2$, hence it holds for $\bar{\beta}$ and $\hat{\beta}$ as well. This proves the lemma.
\end{proof}

\section*{Proof of Lemma~\ref{lemma:ebound}}
\begin{proof}
	\begin{align*}
	\mathbb{P}\Big( \| e \|^2 \geq \sigma^2 \frac{n}{1-\epsilon} \Big) &= \mathbb{P}\Big(\exp\big(\frac{\lambda}{2} \| \frac{e}{\sigma} \|^2\big) \geq \exp\big(\frac{\lambda}{2}  \frac{n}{1-\epsilon}\big)\Big)\\
	& \leq \frac{\mathbb{E}\Big[\exp\big(\frac{\lambda}{2} \|\frac{e}{\sigma} \|^2\big)\Big]}{\exp\big(\frac{\lambda}{2} \frac{n}{1-\epsilon}\big)}\\
	&=\exp\Big(\frac{-\lambda}{2} \frac{n}{1 -\epsilon}\Big) \Big(\frac{1}{1-\lambda}\Big)^{\frac{n}{2}} \ .
	\end{align*} 
	\noindent
	The first equality holds for any $\lambda > 0$, we take $0 < \lambda < 1$ . The second inequality comes from Markov's inequality. The last equality follows since $\frac{e_i}{\sigma} \stackrel{iid}{\sim} \mathcal{N}(0, 1)$. Now, by taking $\lambda = \epsilon$,
	\begin{align*}
	\mathbb{P}\Big(\| e \|^2 \geq \sigma^2 \frac{n}{1-\epsilon}\Big) &\leq \exp\Big(\frac{-n}{2} \big(\frac{\epsilon}{1 -\epsilon} + \log (1-\epsilon)\big)\Big)\\
	&\leq \exp\Big(-\frac{\epsilon^2 n}{4}\Big) \ .
	\end{align*}
	\noindent
	The last inequality holds because for $0 < \epsilon < 1$, $\frac{\epsilon}{1 -\epsilon} + \log (1-\epsilon) \geq \frac{\epsilon^2}{2}$. This proves our lemma,
	\begin{align*} 
	\mathbb{P}\Big(\| e \|^2 \leq \sigma^2 \frac{n}{1-\epsilon}\Big) &\geq 1 - \exp\Big(-\frac{\epsilon^2 n}{4}\Big) \ .
	\end{align*}
\end{proof}

\section*{Proof of Lemma~\ref{lemma:fano}}
\begin{proof}
	Using Fano's inequality~\cite{Cover06}, we can say that,
	\begin{align*}
	%\label{eq:fano}
	\mathbb{P}(\hat{\beta} \neq \bar{\beta}) &\geq 1 - \frac{\mathbb{I}(\bar{\beta}, \mathcal{S}) + \log 2}{\log |\mathcal{F}|} \\
	&\geq 1 - \frac{\mathbb{I}(\bar{\beta}, \mathcal{S}') + \log 2}{\log |\mathcal{F}|} \\
	&\geq 1 - \frac{5n + \log 2}{\log |\mathcal{F}|} \ .
	\end{align*}
	\noindent
	The first inequality follows from equation~\eqref{eq:data processing} and the second inequality follows from the upper bound on the mutual information established in equation~\eqref{eq:mutual info bound}. Now, we want $\mathbb{P}(\hat{\beta} \neq \bar{\beta}) \leq \frac{1}{2}$, then it follows that $n$ must be,
	\begin{align}
	\label{eq:proof fano}
	n \geq \frac{1}{10} \log |\mathcal{F}| - \frac{1}{5} \log 2 \ .
	\end{align}
	\noindent
	This proves the lemma.
	
\end{proof}

\section*{Proof of Theorem~\ref{thm:mainresult}}
\begin{proof}
\item
\paragraph{Constructing an underlying graph $G$} We assume that our underlying graph $G$ fulfills all the properties mentioned while proving Theorem~\ref{thm:noisymainresult}. On this underlying graph $G$, we define our restricted ensemble $\mathcal{F}$ as:
\begin{align*}
\mathcal{F} &= \left\lbrace \beta_i \in \{ 1, -1\},\ if\ i \in S,\ else\ \beta_i = 0,\ S \in \mathbb{M} \right\rbrace \ ,
\end{align*} 
where $\mathbb{M}$ is as in Definition~\ref{def:WGM}.

\paragraph{Bound on the mutual information} We will assume that the elements of design matrix $X$ have been chosen at random and independently from $\mathcal{N}(\frac{1}{s \sqrt{2}}, \frac{1}{s})$. As in the proof of Theorem~\ref{thm:noisymainresult}, we can describe noiseless linear prediction problem as the following Markov's chain:
\begin{align}
\label{eq:markov noiseless}
\bar{\beta} \rightarrow y = X \bar{\beta} \rightarrow z = f (y ) \rightarrow \hat{\beta} \ .
\end{align} 
\noindent
Lets say $\mathcal{S}$ contains $n$ i.i.d. observations of $z$ and $\mathcal{S}'$ contains $n$ i.i.d. observations of $y$. Then using the data processing inequality~\cite{Cover06}, we can say that,
\begin{align}
\label{eq:data processing noiseless}
\mathbb{I}(\bar{\beta}, \mathcal{S}) \leq \mathbb{I}(\bar{\beta}, \mathcal{S}') \ .
\end{align}
\noindent 
Hence, for our purpose it suffices to have an upper bound on $\mathbb{I}(\bar{\beta}, \mathcal{S}')$. Now using results from~\cite{Yu97}, 
\begin{align*}
\mathbb{I}(\bar{\beta}, \mathcal{S}') \leq \frac{1}{|\mathcal{F}|^2} \sum_{\beta \in \mathcal{F}}^{} \sum_{\beta' \in \mathcal{F}}^{} \mathbb{KL}(\mathbb{P}_{\mathcal{S}'|\beta} \| \mathbb{P}_{\mathcal{S}'|\beta'} ) \ .
\end{align*}
where $\mathbb{KL}$ is the Kullback-Leibler divergence. Note that $\mathcal{S}'$ consists of $n$ i.i.d. observations of $y$. Hence,
\begin{align}
\label{eq:bound info noiseless}
\mathbb{I}(\bar{\beta}, \mathcal{S}') \leq \frac{n}{|\mathcal{F}|^2} \sum_{\beta \in \mathcal{F}}^{} \sum_{\beta' \in \mathcal{F}}^{} \mathbb{KL}(\mathbb{P}_{y_i|\beta} \| \mathbb{P}_{y_i|\beta'} ) \ .
\end{align}
\noindent
Furthermore from equation~\eqref{eq:markov noiseless} and noting that the elements of $X$ come independently from $ \mathcal{N}(\frac{1}{s\sqrt{2}}, \frac{1}{s})$,
\begin{align*}
& y_i =  X_i \beta \\
& y_i|\beta \sim \mathcal{N}(\frac{\sum_{k=1}^{d}\beta_k}{s \sqrt{2}} , 1)  \\
& y_i|\beta' \sim \mathcal{N}(\frac{\sum_{k=1}^{d}\beta'_k}{s\sqrt{2}}, 1) \ .
\end{align*}
We can bound $\mathbb{KL}(\mathbb{P}_{y_i|\beta} \| \mathbb{P}_{y_i|\beta'})$ by,
\begin{align*}
\mathbb{KL}(\mathbb{P}_{y_i|\beta} \| \mathbb{P}_{y_i|\beta'}) &= \frac{1}{2}( \frac{\sum_{k=1}^{d} (\beta_k - \beta'_k) }{s\sqrt{2}}  )^2 \\
&\leq 1 \ .
\end{align*}
Substituting $\mathbb{KL}(\mathbb{P}_{y_i|\beta} \| \mathbb{P}_{y_i|\beta'})$ in equation~\eqref{eq:bound info noiseless} we get,
\begin{align}
\label{eq:mutual info noiseless}
\mathbb{I}(\bar{\beta}, \mathcal{S}') \leq  n \ .
\end{align}

\paragraph{Bound on $|\mathcal{F}|$} Using a similar counting logic used in Theorem~\ref{thm:noisymainresult}, we can get:
\begin{align}
\label{eq:bound on F noiseless}
|\mathcal{F}| &\geq 2^s (\frac{d}{g})^g (\frac{\rho(G)Bg}{2(s-g)^2})^{(s-g)} \ .
\end{align} 
\noindent
We prove the theorem by substituting the mutual information from equation~\eqref{eq:mutual info noiseless} and $|\mathcal{F}|$ from equation~\eqref{eq:bound on F noiseless} in the Fano's inequality~\cite{Cover06}.

\end{proof}

\section*{Discussion on the requirements for the underlying graph $G$}
We mentioned before that R1, R2 and R3 are quite mild requirements on the parameters. In fact, it is easy to see that,
\begin{prop}
\label{prop: mild requirements}
Given any value of $s, g$ and $B \geq s - g$, there are infinitely many choices for $\rho(G)$ and $d$ that satisfy R1 and R2 and hence, there are infinitely many $(G, s, g, B)$-WGM which follow our construction. 
\end{prop}
\begin{proof}
	R3 is readily satisfied if each edge has at least unit edge weight and we are not forced to choose isolated nodes in support. Most of the graph-structured sparsity models fulfill this requirement. R2 gives us a lower bound on the choice of $\rho(G)$,
	\begin{align*}
	\rho(G) \geq \frac{2(s-g)^2}{Bg} \ .
	\end{align*}
	Similarly, given a value of $\rho(G)$, R1 just provides a lower bound on choice of $d$,
	 \begin{align*}
	 	d \geq \frac{g \rho(G)B}{s-g} + g \ .
	 \end{align*}
	 Clearly, there is an infinite number of combinations for $\rho(G)$ and $d$.  
\end{proof}

\end{document}